\documentclass[final,a4paper]{article} 
\usepackage{amsthm}
\usepackage{amsmath}
\usepackage{amssymb}
\usepackage{esint}
\usepackage{amsmath,amssymb,amsfonts}
\usepackage{color}
\usepackage{verbatim}
\usepackage{eepic,epic,epsfig,epstopdf}
\usepackage[]{graphics}
\usepackage{graphicx,inputenc}
\usepackage{graphicx,subfig}
\usepackage{bm,color,url}
\usepackage{algorithm,algorithmic}
\usepackage{xr}
\usepackage{float}
\usepackage{setspace}

\newtheorem{lemma}{Lemma}[section]
\newtheorem{theorem}{Theorem}[section]

\newtheorem{remark}{Remark}[section]
\theoremstyle{definition}

\newcommand{\algref}[1]{Algorithm~\textup{\ref{#1}}}

\newcommand{\figref}[1]{Fig.~\textup{\ref{#1}}} 
\newcommand{\lemref}[1]{Lemma~\textup{\ref{#1}}}

\newcommand{\secref}[1]{Section~\textup{\ref{#1}}}
\newcommand{\tabref}[1]{Table~\textup{\ref{#1}}}
\newcommand{\thmref}[1]{Theorem~\textup{\ref{#1}}}
\newcommand{\by}{\mbox{\bf y}}

\newcommand{\bX}{{\bf X}}
\newcommand{\bx}{\mbox{\bf x}}

\newcommand{\bh}{\mbox{\bf h}}

\newcommand{\bb}{\Delta}

\newcommand{\bd}{\mbox{\bf d}}

\newcommand{\bbeta}{\boldsymbol{\beta}}

\def\hat{\widehat}

\def\bar{\overline}

\usepackage{geometry}
\title{A Support Detection and Root Finding Approach for Learning  High-dimensional Generalized Linear Models}
\author{Jian Huang\thanks{Department of Statistics and Actuarial Science, University of Iowa, Iowa City, IA 52242 (jian-huang@uiowa.edu)}
\and  Yuling Jiao\thanks{School of Mathematics and Statistics, Wuhan University, Wuhan 430072,  China. (yulingjiaomath@whu.edu.cn)}
\and Lican Kang \thanks{School of  Mathematics and Statistics,
Wuhan University, Wuhan 430072, China.
(kanglican@whu.edu.cn)}
\and Jin Liu\thanks{Center of Quantitative Medicine Duke-NUS Medical School, Singapore. (jin.liu@duke-nus.edu.sg)}
\and Yanyan Liu\thanks{School of Mathematics and Statistics, Wuhan University, Wuhan 430072, China.
(e-mails: liuyy@whu.edu.cn)}
\and Xiliang Lu\thanks{School of Mathematics and Statistics, Wuhan University, Wuhan 430072,  China, and
Hubei Key Laboratory of Computational Science (Wuhan University), Wuhan, 430072, China. (xllv.math@whu.edu.cn)}
}
\begin{document}

\maketitle

\begin{abstract}
Feature  selection is important for modeling   high-dimensional data, where the number of variables can be much larger than the sample size.
In this paper, we develop a
 support detection and root finding  procedure to learn the high dimensional    sparse generalized linear
 models and denote this method by GSDAR.
Based on the   KKT condition for $\ell_0$-penalized maximum likelihood estimations,
GSDAR generates a sequence of estimators  iteratively.
 Under some restricted invertibility   conditions on the maximum likelihood function and sparsity assumption on the target  coefficients,  the errors of the proposed estimate  decays exponentially
 to the optimal order. Moreover, the oracle estimator can be  recovered  if the target  signal is stronger than the detectable level.
 We conduct simulations and real data analysis to illustrate the advantages of our proposed method over several existing methods, including Lasso and MCP.
\bigskip
\noindent
\textbf{Running title: GSDAR}
\end{abstract}

\section{Introduction}
\label{sec:intro}
In generalized linear models (GLMs) \cite{nelder1972generalized, mccullagh2019generalized}, the
response variable $Y$  follows an  exponential family distribution with density $f(y; \theta) = \exp[y\theta - c(\theta)+ d(y)]$,
where $c(\cdot)$ and $d(\cdot)$ are known functions,  $\theta = \bx^{T}\bbeta^*$, $\bx$ and $\bbeta^*$ represent the $p$-dimension vectors of predictors and the target regression coefficients, respectively.
Let  
$E(y_{i})=\mu_{i}$,
where $\mu_{i}$ is some function of $\theta_{i}=\bx_{i}^T\bbeta$.

When the number of predictors $p$ exceeds the number of sample size $n$, it is often reasonable to assume that   the model is sparse in the sense that there are only small portion  of significant predictors. In this case,  one may  estimate $\bbeta^*$ by the following  $\ell_0$ minimization problem
\begin{equation}\label{eq1}
\begin{split}
&\underset{\bbeta\in \mathbb{R}^p}{\mbox{min}}~ \mathcal{L}(\bbeta)\\
&\mbox{subject~to}~~\|\bbeta\|_0\leq s,
\end{split}
\end{equation}
where $ \mathcal{L}(\bbeta) = - \frac{1}{n}\sum_{i=1}^{n}\big{[}y_{i}\bx_{i}^{T}\bbeta-c(\bx_{i}^{T}\bbeta)+d(y_{i})\big{]} $ is the negative log likelihood function,  $\|\bbeta\|_0$ is defined as the number of nonzero elements of $\bbeta$, and $s>0$ is a tuning parameter that controls the sparsity level. Due to the computational difficulty of solving \eqref{eq1}, many researchers have proposed other penalized   methods for variable selection and estimation  in high-dimensional GLMs.
\cite{park2007l1, van2008high} extended the Lasso method \cite{tibshirani1996regression}
from linear regression to GLMs.  \cite{meier2008group} proposed the group lasso for logistic regression.
\cite{friedman2010regularization} developed coordinate descent to solve the elastic net \cite{zou2005regularization}   penalized   GLMs.
Path following proximal gradient descent \cite{nesterov2013gradient} was adopted in    \cite{wang2014optimal, loh2015regularized} to solve the SCAD \cite{fan2001variable} and MCP
\cite{zhang2010nearly} regularized GLMs.
In \cite{DCN}, the authors propose a DC proximal Newton (DCPN)
method to solve GLMs  with nonconvex sparse promoting penalties such as  MCP/SCAD.
Recently,  \cite{wang2019fast, yuan2017gradient, shen2017iteration} considered Newton type algorithm  for solving sparse GLMs.

In this paper, we propose an approach to variable selection and estimation in high-dimensional  GLMs named GSDAR by a nontrivial extension  of  the support detection  and  rooting finding (SDAR) algorithm \cite{huang2018constructive} which is proposed  to solve  linear regression models and  can not be applied to analyze binary data, categorical variables  in  GLMs.
GSDAR  is  a computational algorithm motivated from the KKT conditions for the Lagrangian version of \eqref{eq1}. It
generates a sequence of solutions  $\{\bbeta^k\}_k$  iteratively, based on support detection using primal and dual information and root finding.
 Under some certain conditions on $\mathcal{L}$ and sparsity assumptions on the regression coefficient $\bbeta^*$,  we prove that the estimation errors   decay exponentially to the optimal order. Moreover, the oracle estimator can be  recovered  with high probability  if the target  signal is over the detectable level.

The rest of this paper is organized as follows. In   \secref{GSDAR}, we present the  detail derivation of GSDAR algorithm.
In \secref{Theorem}, we bound the estimation error of  GSDAR.
In \secref{AGSDAR}, we extend GSDAR algorithm to AGSDAR, an  adaptive version of GSDAR.
In \secref{simulation}, we demonstrate GSDAR and AGSDAR on   the simulation and real data via comparing with state-of-the-art methods.
We conclude in \secref{conclusion}. Proofs for all the lemmas and theorems are provided in the Appendix.
\section{Derivation of GSDAR}\label{GSDAR}
First, we  introduce  some  notations used throughout the paper. We write $n\gtrsim \log(p)$ to mean that $n\geq c\log(p)$ for some universal constant $c\in (0,\infty)$.
Let $\|\bbeta\|_q=(\sum_{i=1}^{p}|\beta_{i}|^q)^{\frac{1}{q}}$  denote the  $q$ ($q\in [1,\infty]$) norm   of a vector $\bbeta=(\beta_1,...,\beta_p)^{T}\in \mathbb{R}^{p}$.
Let supp$(\bbeta)$=$\{i:\beta_i\neq0,~i=1,...,p\}$ denote the support of $\bbeta$, and $A^* = \mathrm{supp} (\bbeta^*).$
Let $|A|$ denote the length of the set $A$ and  denote $\bbeta_A=(\beta_i,i\in A)\in \mathbb{R}^{|A|}$, $\bbeta|_A\in \mathbb{R}^{p}$ with its $i\text{-}$th element $({\bbeta|_A})_i=\beta_i 1(i\in A)$, where $\textbf{1}(\cdot)$ is the indicator function. Denote $\bX_A=(\bx_j,j\in A)\in \mathbb{R}^{n\times |A|}$, where $\bx_j$  is $j$-th column of the covariate matrix $\bX\in \mathbb{R}^{n\times p}$.
$\|\bbeta\|_{T,\infty}$ and $\|\bbeta\|_{\min}$ denote the $T$-th largest elements (in absolute value) and the minimum absolute value of $\bbeta$, respectively.
$\nabla \mathcal{L}$ and $\nabla^2 \mathcal{L}$ denote the gradient  and Hessian  of function $\mathcal{L}$, respectively.

The Lagrangian form  of \eqref{eq1} is
\begin{equation}\label{eq3}
\underset{\bbeta\in \mathbb{R}^p}{\mbox{min}}~\mathcal{L}(\bbeta)+\lambda\|\bbeta\|_0.
\end{equation}
By similar arguments as Lemma 1 of \cite{huang2018constructive},
we obtain the following KKT condition of \eqref{eq3}.
\begin{lemma}\label{L3.1}
If $\bbeta^\diamond$ is a minimizer of \eqref{eq3}, then $\bbeta^\diamond$ satisfies:
\begin{equation}\label{eq4}
\left\{
\begin{aligned}
&\bd^{\diamond}=-\nabla \mathcal{L}(\bbeta^{\diamond}),\\
&\bbeta^{\diamond}=H_{\lambda}(\bbeta^{\diamond}+\bd^{\diamond}),
\end{aligned}
\right.
\end{equation}
where the $i$-th element of $H_{\lambda}(\cdot)$ is defined by
\begin{equation*}
 (H_{\lambda}(\bbeta))_{i}=\left\{
\begin{aligned}
&0  ,&&&|\beta_{i}|\leq \sqrt{2\lambda},\\
&\beta_{i}  ,&&&|\beta_{i}|\geq\sqrt{2\lambda}.\\
\end{aligned}\right.
\end{equation*}
Conversely, if $\bbeta^\diamond$ and $\bd^{\diamond}$ satisfy \eqref{eq4}, then $\bbeta^\diamond$ is a local minimizer of \eqref{eq3}.
\end{lemma}
\begin{proof}
See Appendix   A.
\end{proof}

Let $A^\diamond=\text{supp}(\bbeta^\diamond)$, $I^\diamond=(A^\diamond)^c$. From the definition of $H_{\lambda}(\cdot)$ and \eqref{eq4}, we can conclude that
\begin{equation*}
A^\diamond=\{i:|\beta^\diamond_{i}+d^\diamond_{i}|\geq\sqrt{2\lambda}\},\quad I^\diamond=\{i:|\beta^\diamond_{i}+d^\diamond_{i}|<\sqrt{2\lambda}\},
\end{equation*}
and
$$\left\{
\begin{aligned}
&\bbeta_{I^\diamond}^{\diamond} =\textbf{0}\\
&\bd_{A^\diamond}^{\diamond} =\textbf{0}\\
&\bbeta_{A^\diamond}^{\diamond}\in\underset{\bbeta_{A^\diamond}}{\mbox{argmin}}~\widetilde{\mathcal{L}}(\bbeta_{A^\diamond})\\
&\bd_{I^\diamond}^{\diamond}=[-\nabla \mathcal{L}(\bbeta^\diamond)]_{I^\diamond},
\end{aligned}
\right.
$$
where
\begin{align*}
\widetilde{\mathcal{L}}(\bbeta_{A^\diamond})&=\mathcal{L}(\bbeta|_{A^\diamond})\\
&=-\frac{1}{n}\sum_{i=1}^{n}\left[y_{i}\bx_{i{(A^\diamond)}}^{T}\bbeta_{A^\diamond}-c\big{(}\bx_{i{(A^\diamond)}}^{T}\bbeta_{A^{\diamond}}\big{)}+d(y_{i})\right].
\end{align*}
If $\{\bbeta^{k},\bd^{k}\}$ can approximate $\{\bbeta^{\diamond},\bd^{\diamond}\}$ well, then $\{A^{k},I^{k}\}$  can also approximate $\{A^{\diamond},I^{\diamond}\}$ well, where $\{A^{k},I^{k}\}$ is expressed as
\begin{equation}\label{eq6}
A^k=\{i:|\beta^k_{i}+d^k_{i}|\geq\sqrt{2\lambda}\},\quad I^k=\{i:|\beta^k_{i}+d^k_{i}|<\sqrt{2\lambda}\}.
\end{equation}
Thus we get a new approximation pair $\{\bbeta^{k+1}_{I^{k}},\bd^{k+1}_{A^{k}},\bbeta^{k+1}_{A^{k}},\bd^{k+1}_{I^{k}}\}$ showed as follow:
\begin{equation}\label{eq7}
\left\{
\begin{aligned}
&\bbeta^{k+1}_{I^{k}}=\textbf{0}\\
&\bd^{k+1}_{A^{k}} =\textbf{0}\\
&\bbeta^{k+1}_{A^{k}}\in\underset{\bbeta_{A^k}}{\mbox{argmin}}~\widetilde{\mathcal{L}}(\bbeta_{A^k})\\
&\bd^{k+1}_{I^{k}}=[-\nabla \mathcal{L}(\bbeta^{k+1})]_{I^k},\\
\end{aligned}
\right.
\end{equation}
where
\begin{align*}
\widetilde{\mathcal{L}}(\bbeta_{A^k})&=\mathcal{L}(\bbeta|_{A^k})\\
&=-\frac{1}{n}\sum_{i=1}^{n}\left[y_{i}\bx_{i{(A^k)}}^{T}\bbeta_{A^k}-c\big{(}\bx_{i{(A^k)}}^{T}\bbeta_{A^{k}}\big{)}+d(y_{i})\right]
\end{align*}
If we have the prior information that $\parallel\bbeta^*\parallel_{0}=K\leq T$, then we set
\begin{equation}\label{eq8}
\sqrt{2\lambda}=\parallel\bbeta^k+\bd^k\parallel_{T,\infty}
\end{equation}
 in \eqref{eq6}. Thus $|A^k|=T$ in every iteration due to this $\lambda$. Let $\bbeta^0$ be an initial value, then we get a sequence of solutions $\{\bbeta^k,k\geq1\}$ by using \eqref{eq6}
 and \eqref{eq7} with the $\lambda$ in \eqref{eq8}.

The GSDAR algorithm is described in \algref{algorithm1}.
\begin{algorithm}[!ht]
\caption{GSDAR}
\label{algorithm1}
\begin{algorithmic}[1]
\STATE Input:
$\bbeta^0$, $T$, $\bd^0=-\nabla \mathcal{L}(\bbeta^{0})$; $k=0$
\FOR{$k= 0,1,\ldots,$}
\STATE $A^k=\big{\{}j:|\beta^k_j+d^k_{j}|\geq\|\bbeta^k+\bd^k\|_{T,\infty}\big{\}}$, $I^{k}=(A^{k})^c$.
\STATE $\bbeta^{k+1}_{I^k}=\textbf{0}$.
\STATE $\bd^{k+1}_{A^k}=\textbf{0}$.
\STATE $\bbeta^{k+1}_{A^k}=\underset{\bbeta_{A^k} }{\mbox{argmin}}~\widetilde{\mathcal{L}}(\bbeta_{A^k})$.
\STATE $\bd^{k+1}_{I^{k}}=[-\nabla \mathcal{L}(\bbeta^{k+1})]_{I^k}$.
\STATE\textbf{if} $A^{k}=A^{k+1}$, \textbf{then}
\STATE \quad Stop and denote the last iteration $\bbeta_{\hat {A}}$, $\bbeta_{\hat {I}}$, $\bd_{\hat {A}}$, $\bd_{\hat {I}}$.
\STATE \textbf{else}
\STATE \quad $k=k+1$
\STATE \textbf{end if}
\ENDFOR
\STATE Output: $\hat{\bbeta}=(\bbeta^{\mathrm{T}}_{\hat {A}}$, $\bbeta^{\mathrm{T}}_{\hat {I}})^{\mathrm{T}}$ as the estimates of $\bbeta^*$.
\end{algorithmic}
\end{algorithm}

In  \algref{algorithm1}, we terminate  GSDAR when $A^{k}=A^{k+1}$ for some $k$, because the sequences generated by GSDAR will not change. In \secref{Theorem}, we will prove that under some regularity certain conditions on $\bX$ and $\bbeta^*$, with high probability
$A^*=A^{k}=A^{k+1}$ in finite steps, i.e., the GSDAR will stop and whence the oracle estimator will be  recovered.
\section{Theoretical Properties}\label{Theorem}
In this section, we will give the
$\ell_{\infty}$  error bounds for the GSDAR estimator.
Under some certain  conditions, we show that $\|\bbeta^{k}-\bbeta^{*}\|_{\infty}$ achieves sharp estimation  error.
Furthermore, if the minimum value of  target signal is detectable, GSDAR will get the oracle estimator in finite steps if
$K$ is chosen just as the true model size $T$.
We first introduce the following  restricted invertibility  conditions.
\begin{enumerate}
\item[(C1)]\label{C1} There exist  constants $0<L < U\in (0,\infty)$ such that, for all different vectors $\bbeta_{1}$ and $\bbeta_{2}$ with $\|\bbeta_{1}-\bbeta_{2}\|_{0}\leq 2T$,
    \begin{equation*}
    0<L\leq\frac{(\bbeta_{1}-\bbeta_{2})^{T}\cdot\nabla^{2}\mathcal{L}(\widetilde{\bbeta})\cdot(\bbeta_{1}-\bbeta_{2})}{\|\bbeta_{1}-\bbeta_{2}\|_1\|\bbeta_{1}-\bbeta_{2}\|_{\infty}}\leq U< \infty,
    \end{equation*}
where $\widetilde{\bbeta}=\bbeta_1+\nu(\bbeta_2-\bbeta_1)$ for any $\nu \in (0,1)$.
\item[(C2)]
$\|\bbeta^*_{A^*}\|_{\min}\geq \frac{3c_1}{L} \sqrt{\frac{\log(p)}{n}}$, where $c_1$ is a universal numerical constant.
\end{enumerate}
\begin{remark}
Condition (C1) extends the  the weak cone  invertibility condition   in \cite{ye2010rate}. This kind restricted strong convexity  type regularity condition is needed  in bounding the estimation error in high dimension statistics \cite{zhang2012general}.
Condition (C2) is  required  to guarantee   the target signal to be  detectable  in high dimension linear regressions.
\end{remark}
\subsection{$\ell_{\infty}$ error bounds}
\begin{theorem}\label{th1}
Assume (C1) holds with  $0<U<\frac{1}{T}$.  Set $K\leq T$ and $\bbeta^0 = \textbf{0}$ in  \algref{algorithm1}.
\item(i) Before \algref{algorithm1} terminates, we have
\begin{align*}
&\|\bbeta^{k}-\bbeta^{*}\|_{\infty}\\
&\leq
\sqrt{(K+T)(1+\frac{U}{L})}(\sqrt{\xi})^k\|\bbeta^*\|_{\infty}+\frac{2}{L}\|\nabla \mathcal{L}(\bbeta^*)\|_{\infty}.
\end{align*}
where $\xi=1-\frac{2L(1-TU)}{T(1+K)}\in(0,1)$.
\item(ii) Assume the rows of  $\bX$  are i.i.d. sub-Gaussian with  $n\gtrsim \log(p)$, then there exists universal constants $(c_1, c_2, c_3)$ with $0<c_i<\infty$, $i=1,2,3$, such that  with probability at least $1-c_2\exp(-c_3\log(p))$,
\begin{align*}
&\|\bbeta^{k}-\bbeta^{*}\|_{\infty}\\
&\leq
\sqrt{(K+T)(1+\frac{U}{L})}(\sqrt{\xi})^k\|\bbeta^*\|_{\infty}
+\frac{2c_1}{L} \sqrt{\frac{\log(p)}{n}},
\end{align*} i.e.,
\begin{equation*}
\|\bbeta^{k}-\bbeta^{*}\|_{\infty}\leq \mathcal{O}\left(\sqrt{\frac{\log(p)}{n}}\right)
\end{equation*}
with high probability if $k \geq \mathcal{O}\left(\log_{\frac{1}{\xi}}\frac{n}{\log(p)}\right).$
\end{theorem}
\begin{proof}
See Appendix   B.
\end{proof}
\begin{remark}
The requirement $U < \frac{1}{T}$  is not essential since we can always rescale  the loss function  $\mathcal{L}$   to make it hold. This rescaling is  equivalent to  multiplying  a  step size  to the dual variable in the   the GSDAR algorithm.
Let $\tau$ be this step size satisfying
$0<\tau <\frac{1}{TU}$.
Then,  \thmref{th1} still holds by replacing  $\xi$ with $1-\frac{2\tau L(1-\tau TU)}{T(1+K)}\in(0,1)$.

Before we submit this work, We aware that \cite{wang2019fast} proposed the sparse Newton method to solve   high dimensional logistic regression. The sparse Newton algorithm is similar to  GSDAR with step size.  However, \cite{wang2019fast} proved   a   fast local  convergence result of  $\bbeta^k$  to the minimizer $\bbeta^{\diamond}$  from the point view of optimization. Here, we bound   the  estimation error of  $\bbeta^k$ to the target $\bbeta^*$ from the angle of statistics.
\end{remark}
\subsection{Support recovery}
\begin{theorem}\label{th3}
Assume (C1) and (C2) hold with  $0<U<\frac{1}{T}$, and the  rows of $\bX$ are i.i.d. sub-Gaussian with  $n\gtrsim \log(p)$.  Set $K\leq T$ in  \algref{algorithm1}.
Then   with probability at least $1-c_2\exp(-c_3\log(p))$,
$A^*\subseteq A^k$ if $k> \log_{\frac{1}{\xi}} 9 (T+K)(1+\frac{U}{L})r^2$, where $r = \frac{\|\bbeta^*\|_{\infty} }{\|\bbeta^*_{A^*}\|_{\min}}$ is the range of $\bbeta^*$.
\end{theorem}
\begin{proof}
See Appendix   C.
\end{proof}
\begin{remark}

 \thmref{th3}  demonstrates that the estimated support via GSDAR
can cover the true support with the cost at most   $\mathcal{O}(\log (T))$ number of iteration if the minimum signal strength of $\bbeta^*$ is above the detectable threshold  $ \mathcal{O}(\sqrt{\frac{\log(p)}{n}})$. Support recovery for sparse  GLMs has also been  studied in
\cite{DCN,yuan2017gradient, shen2017iteration}.
In \cite{DCN}, the authors propose a DC proximal Newton (DCPN)
method to solve GLMs  with nonconvex sparse promoting penalties such as  MCP/SCAD.
 They derive an    estimation error in $\ell_2$ norm   with order $\mathcal{O}(\sqrt{\frac{K\log p}{n}})$   under similar assumptions as that of our (C1). And  they show that the true  support can be reconverted  under the requirement $\|\bbeta^*_{A^*}\|_{\min}\geq \mathcal{O}(\sqrt{\frac{K\log(p)}{n}}),$ which is stronger than  our assumption (C2). The computational complexity of DCPN is  worse than GSDAR  since  the DCPN is    based on  the multistage convex relaxation scheme
to transform the original nonconvex  optimizations into
sequences of LASSO regularized GLMs, therefore,  a Lasso inner solver is called
at each stage \cite{picaso}.
\cite{yuan2017gradient, shen2017iteration}. They proved that  Gradient Hard Thresholding
Pursuit  can recover the true  support under the requirement $\|\bbeta^*_{A^*}\|_{\min}\geq \mathcal{O}(\sqrt{\frac{K\log(p)}{n}}),$ which is stronger than  our assumption (C2).

Further, if we set $T = K$ in GSDAR, then the stopping condition $A^k = A^{k+1}$ will hold if $k \geq  \mathcal{O}(\log (K))$ since the estimated supports coincide with the true support.  
As a consequence, the oracle estimator will be recovered in    $\mathcal{O}(\log (K))$ steps.  Neither in  \cite{yuan2017gradient} nor in \cite{shen2017iteration} proved that the stopping condition of Gradient Hard Thresholding
Pursuit  can be  satisfied.
 Meanwhile, the iteration complexity of Gradient hard thresholding pursuit  analyzed by \cite{shen2017iteration} is  $\mathcal{O}(K)$, which is worse   than the complexity bound established here.
\end{remark}

\section{Adaptive GSDAR}\label{AGSDAR}
In practice, the sparsity level of the true parameter value $\bbeta^*$ is often unknown. As for that, we can regard $T$ as the tuning parameter.
Let $T$ increase from 0 to $Q$, which is a given large enough integer, then we can get a set of solutions paths: $\{\widehat{\bbeta}(T):T=0,1,...,Q\}$, where $\widehat{\bbeta}(0)=0$.
Generally, we can take $Q=\alpha n/\log(n)$ as suggested by \cite{fan2008sure}, where $\alpha$ is a positive and finite constant. We can use some methods such as  the cross-validation or HBIC
\cite{wang2013calibrating} to get $\widehat{T}$, the estimation of T. Thence we can take $\widehat{\bbeta}(\widehat{T})$ as the estimation of $\bbeta^{*}$.

In addition, we can  run  \algref{algorithm1} until the consecutive solutions is smaller than a prespecified tolerance level $\varepsilon$ by increasing $T$.
Also, we can
increase $T$ to run
\algref{algorithm1} until the residual square sum is less than a given tolerate level $\varepsilon$, then output $\bbeta^k$ at this time to terminate the calculation. If the purpose of the model is to classify, we can  stop the calculation  until classification accuracy rate achieve a certain level.
We summarize the  Adaptive GSDAR in  following  \algref{algorithmic2}.

\begin{algorithm}[H]
\caption{AGSDAR}
\label{algorithmic2}
\begin{algorithmic}[1]
\STATE Input:
$\bbeta^0$, $\bd^0$=$-\nabla \mathcal{L}(\bbeta^{0})$, an integer $\vartheta$, an integer Q, an early stopping criterion (optional). Set $k=1$.
\FOR{$k= 0,1,\ldots,$}
\STATE Run  \algref{algorithm1} with $T=\vartheta k$ and with initial value $\bbeta^{k-1}$, $\bd^{k-1}$. Denote the output by $\bbeta^k$, $\bd^k$.
\STATE \textbf{if} the early stopping criterion is satisfied or $T>Q$,
\textbf{then}
\STATE \quad stop
\STATE \textbf{else}
\STATE \quad $k=k+1$
\STATE \textbf{end if}
\ENDFOR
\STATE Output: $\widehat{\bbeta}(\widehat{T})$~as the estimates of $\bbeta^*$.
\end{algorithmic}
\end{algorithm}

\section{Simulation Studies and real data analysis}\label{simulation}
In this section, we make some simulations  and real data analysis in logistic regression model
to illustrate our proposed methods GSDAR and AGSDAR. First, we compare the  simulations results of GSDAR/AGSDAR with Lasso and MCP in terms of  accuracy, efficiency and classification accuracy rate. Then,  we further compare AGSDAR with Lasso and MCP  on the effects of  model parameters such as sample size $n$,  variable dimension $p$
and   correlation $\rho$ in $\bX$.
Third, we get the average iterative steps of GSDAR.
Last, GSDAR and AGSDAR are  compared  with Lasso and MCP  on  some real data sets.

Our implement of Lasso and MCP is according to the R package ncvreg developed by \cite{breheny2011coordinate}. In implement of AGSDAR, we set $Q=n/\log(n)$, and do not use the early stopping criterion instead use HBIC criteria  to chose the $T$.
\subsection{Accuracy, efficiency and classification accuracy rate}
We generate  the design matrix $\bX$ as follows. First, we generate a $n\times p$ random Gaussian matrix $\bar{\bX}$
whose entries are i.i.d. $\sim N(0,1)$,
and normalize its columns to the $\sqrt{n}$ length. Then the design matrix $\bX$ is generated with
$\bx_1=\bar{\bx}_1$, $\bx_p=\bar{\bx}_p$, and $\bx_j=\bar{\bx}_j+\rho(\bar{\bx}_{j+1}+\bar{\bx}_{j-1})$, $j=2,...,p-1$. The underlying regression coefficient $\bbeta^*$ with
$K$ nonzero coefficients is generated
such that the $K$ nonzero coefficients in $\bbeta^*$ are uniformly
distributed in $(m_1,m_2)$, where $m_1=5\sqrt{2\log{p}/n}$
and $m_2=100\cdot m_1$.
Besides, the $K$ nonzero coefficients are  randomly assigned to the $K$ components of $\bbeta^*$.
The responses $y_{i}\sim B(1,p_{i})$,
where $p_{i}$=$\frac{\exp(\bx_{i}^T\bbeta^*)}{1+\exp(\bx_{i}^T\bbeta^*)}$, $i=1,...,n$.

Since Logistic regression model aims to classify, we randomly choose $80\%$ of the samples as the training set and the rest for the test set to get the  classification accuracy rate by predicting. Set $n=300$, $p=5000$, $K=10$ and $\rho=0.2:0.2:0.8$.
\begin{table}[!ht]
\centering
\caption{Numerical results (the averaged relative error, CPU time, the average  classification accuracy rate by predicting) on data set with $n=300$, $p=5000$, $K=10$, $\rho=0.2:0.2:0.8$.}
\label{table1}
\scalebox{1.4}{
\begin{tabular}{cccccccccc}
\hline
 $\rho$  & method & ReErr & Time(s) & ACRP \\
\hline
     &Lasso& 0.99&6.03&86.68\%\\
 0.2&MCP&0.95&11.93&93.95\%\\
 &GSDAR&0.69&0.60&92.62\%\\
  & AGSDAR & 0.95&1.42&91.15\%\\
\hline
   &Lasso&0.99& 6.11 & 86.62\%\\
 0.4&MCP&0.95&11.07&94.37\%\\
&GSDAR&0.69&0.64&92.47\%\\
& AGSDAR&0.97&1.33&88.73\%\\
\hline
    &Lasso&0.99&6.33&86.55\% \\
  0.6 &MCP &0.96&11.47&93.85\%\\
  &GSDAR&0.70&0.55&94.40\%\\
& AGSDAR &0.98& 1.41&89.80\%\\
\hline
    &Lasso&1.00&6.28&86.43\% \\
  0.8 &MCP& 0.97&11.47&93.38\%\\
  &GSDAR&0.79&0.60&96.11\%\\
& AGSDAR &0.98& 1.44&89.75\%\\
\hline
\end{tabular}
}
\end{table}

\tabref{table1} displays simulation results including the average of relative error of estimate  $\hat{\bbeta}$ defined as ReErr=$\frac{1}{100}\sum\|\hat{\bbeta}-\bbeta^*\|/\|\bbeta^*\|$, CPU time and classification accuracy rate of prediction defined as ACRP based on 100 independent replications.

We can conclude that GSDAR has the lowest values in ReErr regardless of the values of $\rho$, while Lasso, MCP and AGSDAR have almost  same values  in ReErr.
In terms of the speed, GSDAR is the fastest among all the considered methods
with 10 times fast to Lasso and 20 times fast to MCP for every $\rho$.
AGSDAR is also significantly faster than Lasso and MCP,
and its speed is nearly $5$ and $8$ times that of Lasso and MCP, respectively. As for the average classification accuracy rate, GSDAR has higher classification accuracy rate than other methods when  $\rho > 0.4$, however,  MCP is slightly better than GSDAR  when
$\rho\leq 0.4$.
In summary, GSDAR and AGSDAR perform well in terms of computational speed,  GSDAR can effectively get the oracle estimator and has excellent results in predicting.
\subsection{Influence of the model parameters}
We now consider the effects of each of the model parameters on the performance of AGSDAR, Lasso  and MCP.
We generate the design matrix $X$ by the way  that each row of $X$ comes from $N(0,\Sigma)$, where $\Sigma_{ij}=\rho^{|i-j|}$, $1\leq i,j\leq p$.
 Let $R=m_2/m_1$, where $m_2=
\|\bbeta^*_{A^*}\|_{\max}$ and $m_1=\|\bbeta^*_{A^*}\|_{\min}=1$. The underlying regression coefficient vector $\bbeta^*\in \mathbb{R}^p$ is
generated in such a way that
the $K$ nonzero coefficients in $\bbeta^*$ are uniformly
distributed in $(m_1,m_2)$,
and $A^*$ is a randomly chosen subset of $\{1,...,p\}$ with $|A^*|=K<n$. 
Then the observation variable $y_{i}\sim B(1,p_{i})$,
where $p_{i}$=$\frac{\exp(\bx_{i}^T\bbeta^*)}{1+\exp(\bx_{i}^T\bbeta^*)}$, $i=1,...,n$.

We compare the performance of all the considered methods in terms of
average positive discovery rate (APDR), average false discovery rate (AFDR) and average combined discovery rate (ADR) defined by \cite{luo2014sequential}
to characterize the selection accuracy of different parameters to the model and
showed as follows.
\begin{eqnarray*}
\mbox{APDR}&=&\frac{1}{100}\sum\frac{|\hat{A}\bigcap A^*|}{|A^*|},\\
\mbox{AFDR}&=&\frac{1}{100}\sum\frac{|\hat{A}\bigcap {A^*}^{c}|}{|\hat{A}|},\\
\mbox{ADR}&=&\mbox{APDR}+(1-\mbox{AFDR}),
\end{eqnarray*}
where $\hat{A}$ denotes the estimated support set.
The following simulations  are  based on 100 independent replications.
\subsubsection{Influence of the sample size $n$}
\tabref{table2}  shows the influence of the sample size $n$
on APDR, AFDR and ADR.
We set  $p=500$, $K=6$, $R=10$, $\rho=0.3$ and let $n$ varies from 100 to 400 by step 50 to generate the data.
\begin{table}[!ht]
\centering
\caption{Numerical results (APDR, AFDR, ADR)
on the data $p=500$, $K=6$, $R=10$, $\rho=0.3$
and $n=100:50:400$.}
\label{table2}
\scalebox{1.4}{
\begin{tabular}{ccccccccc}
\hline
 $n$  & method & APDR & AFDR & ADR \\
\hline
   &Lasso&0.83&0.84&0.99\\
100  &MCP&0.79& 0.36&1.43\\
&AGSDAR&0.72&0.19&1.53\\
\hline
    &Lasso&0.92&0.87&1.05\\
   150 &MCP&0.90&0.22&1.68\\
&AGSDAR&0.85&0.15&1.70\\
\hline
    &Lasso&0.95&0.88&1.07\\
200 &MCP & 0.93&0.19&1.74\\
&AGSDAR&0.90&0.12&1.78\\
\hline
    &Lasso&0.97&0.89&1.08\\
250 &MCP&0.93&0.16&1.77\\
&AGSDAR&0.93&0.06&1.87\\
\hline
    &Lasso&0.98&0.89&1.09\\
300 &MCP&0.95&0.15&1.80\\
&AGSDAR&0.96&0.06&1.90\\
\hline
    &Lasso&0.99&0.89&1.10\\
  350&MCP&0.95&0.16&1.79\\
&AGSDAR&0.96&0.05&1.91\\
\hline
   &Lasso&0.99&0.89&1.10\\
  400 &MCP&0.97&0.15&1.82\\
&AGSDAR&0.98&0.05&1.93\\
\hline
\end{tabular}
}
\end{table}

It can be seen  that as the sample size $n$ increases,
Lasso always has  the highest values on  APDR among the three methods. However, Lasso also has the
worst values on AFDR for each $n$, which is only a little smaller than APDR.
It indicates that Lasso tends to choose more variables,
even there are many unsuitable variables being selected. Therefore, Lasso is more greedy in selecting variables than MCP and AGSDAR.
AGSDAR always has  the best values on AFDR and ADR for every $n$, and its values on APDR are also not small, which means that AGSDAR can effectively prevent the erroneous variable from being selected while selecting  as many proper variables as possible into the model, especially when the sample size $n$  is getting larger.
MCP is similar to AGSDAR,
it can not only select a certain amount of proper variables, but also prevent some improper variables from being selected into the model,
while it still chooses more improper variables into the model than AGSDAR.
Hence,  AGSDAR can always select more proper variables effectively and minimize the number of improper variables selected into the model with the increasing sample size $n$.
\subsubsection{ Influence of the variable dimension $p$}
\tabref{table3}  shows the influence of the variable dimension $p$
on the  APDR, AFDR and ADR. We fix
$n=100$, $K=6$, $R=10$, $\rho=0.2$, and set  $p=100:100:700$  to generate the data.
\begin{table}[!ht]
\centering
\caption{Numerical results (APDR, AFDR, ADR)
on the data $n=100$, $K=6$, $R=10$, $\rho=0.2$ and  $p=100:100:700$.}
\label{table3}
\scalebox{1.4}{
\begin{tabular}{ccccccccc}
\hline
$p$&method&APDR&AFDR&ADR\\
\hline     &Lasso&0.92&0.77&1.15\\
100 &MCP &0.83& 0.20& 1.63\\
&AGSDAR&0.82&0.16&1.66\\
\hline
&Lasso&0.88&0.81&1.07\\
 200 &MCP &0.83& 0.23& 1.60\\
&AGSDAR&0.80&0.17&1.63\\
\hline
&Lasso&0.89&0.82&1.07\\
300 &MCP &0.82& 0.29& 1.53\\
&AGSDAR&0.80&0.21&1.59\\
\hline
&Lasso&0.84&0.84&1.00\\
400&MCP & 0.79&0.34&1.45\\
&AGSDAR&0.75&0.20&1.55\\
\hline
&Lasso&0.83&0.85&0.98\\
500&MCP &0.78& 0.35& 1.43\\
&AGSDAR&0.74&0.20&1.54\\
\hline
&Lasso&0.79&0.85&0.94\\
600&MCP&0.77&0.39& 1.38\\
&AGSDAR&0.70&0.22&1.48\\
\hline  &Lasso&0.80&0.85&0.95\\
700&MCP&0.77&0.37&1.40\\
&AGSDAR&0.70&0.25&1.45\\
\hline
\end{tabular}
}
\end{table}

As  \tabref{table3} depicted,
Lasso has the largest values on  APDR and AFDR, and lowest values on  ADR for every variable dimension $p$.
Meanwhile, the values of Lasso on AFDR are  higher than that of APDR when $p>400$ and beyond 0.5 for each $p$, which suggests that Lasso selects much more improper variables than proper variables into model,
thus it increases the complexity of model.
AGSDAR and MCP take almost  same values on APDR, especially when $p<600$,
indicating that  MCP and AGSDAR have the same ability to select proper variables when $p$ takes the appropriate values. Besides, AGSDAR  gets the best values on AFDR and ADR for every $p$. Hence, to the utmost extent,  AGSDAR can  prevent the improper variables being selected into the model, thus reduce the complexity of the model.
\subsubsection{ Influence of the correlation $\rho$}
\tabref{table4}  shows the influence of the correlation $\rho$ on APDR, AFDR and ADR. We set $n=150$, $p=500$, $K=6$, $R=10$
and $\rho=0.1:0.1:0.9$ to generate the data.
\begin{table}[!ht]
\centering
\caption{Numerical results (APDR, AFDR, ADR)
on the data $n=150$, $p=500$, $K=6$, $R=10$ and $\rho=0.1:0.1:0.9$.}
\label{table4}
\scalebox{1.4}{
\begin{tabular}{ccccccccc}
\hline
 $\rho$  & method & APDR & AFDR & ADR \\
\hline

    &Lasso&0.92&0.87&1.05\\
  0.1  &MCP &0.87& 0.22&1.65\\
&AGSDAR&0.85&0.15&1.70\\
	\hline
   &Lasso&0.92& 0.87&1.05\\
 0.2&MCP & 0.89&0.21&1.68\\
&AGSDAR&0.85&0.15&1.70\\
\hline
    &Lasso&0.92&0.87&1.05\\
 0.3 &MCP & 0.90& 0.23&1.67\\
&AGSDAR&0.88&0.13&1.75\\
\hline
    &Lasso&0.91&0.87&1.04\\
  0.4&MCP&0.87&0.23&1.64\\
&AGSDAR&0.84&0.15&1.69\\		
\hline
   &Lasso&0.90 &0.86&1.04\\
  0.5 &MCP & 0.85&0.26&1.59\\
&AGSDAR&0.83&0.16&1.67\\
\hline
    &Lasso&0.90&0.87&1.03\\
  0.6&MCP&0.88&0.22&1.66\\
&AGSDAR&0.84&0.16&1.68\\
\hline
&Lasso&0.90&0.86&1.04\\
  0.7&MCP&0.83&0.26&1.57\\
&AGSDAR&0.80&0.22&1.58\\
\hline
&Lasso&0.88&0.86&1.02\\
  0.8&MCP&0.75&0.31&1.44\\
&AGSDAR&0.75&0.26&1.49\\
\hline
&Lasso&0.82&0.84&0.98\\
0.9&MCP&0.55&0.48&1.07\\
&AGSDAR&0.58&0.44&1.14\\
\hline
\end{tabular}
}
\end{table}

In  \tabref{table4},
Lasso performs similarly
as the first two simulations about the sample size $n$ and the variable dimension $p$
affecting the model.
Lasso also has the  best values on  APDR and  worst values on AFDR and ADR for every $\rho$.
On the one hand,
AGSDAR and MCP have nearly same values on APDR for each $\rho$.
On the other hand,
 with increasing correlation $\rho$, AGSDAR  always obtains  the best values on AFDR and ADR.
  Therefore we can conclude that AGSDAR can simultaneously select a certain number of proper variables and  prevent the improper variables into the model all the time with increasing correlation $\rho$.
\subsection{Number of iterations}
In order to further illustrate the effectiveness of GSDAR, we conduct simulations to get the average number of iterations of GSDAR with $K$=$T$ in  \algref{algorithm1}.
 We generate the data as the same way described in  subsection 5.2.
Meanwhile,  we take the influence of  correlation $\rho$ into account, then
we obtain the average number of iterations for different values of  correlation $\rho$.
\figref{F1} shows the average number of iterations of GSDAR
based 100 independent replications on data set:
$n=500$, $p=1000$, $K= 2 : 2 : 50$, $R=3$ and $\rho=0.1:0.2:0.7$.
\begin{figure}[!ht]
  \centering
  \includegraphics[width=\linewidth]{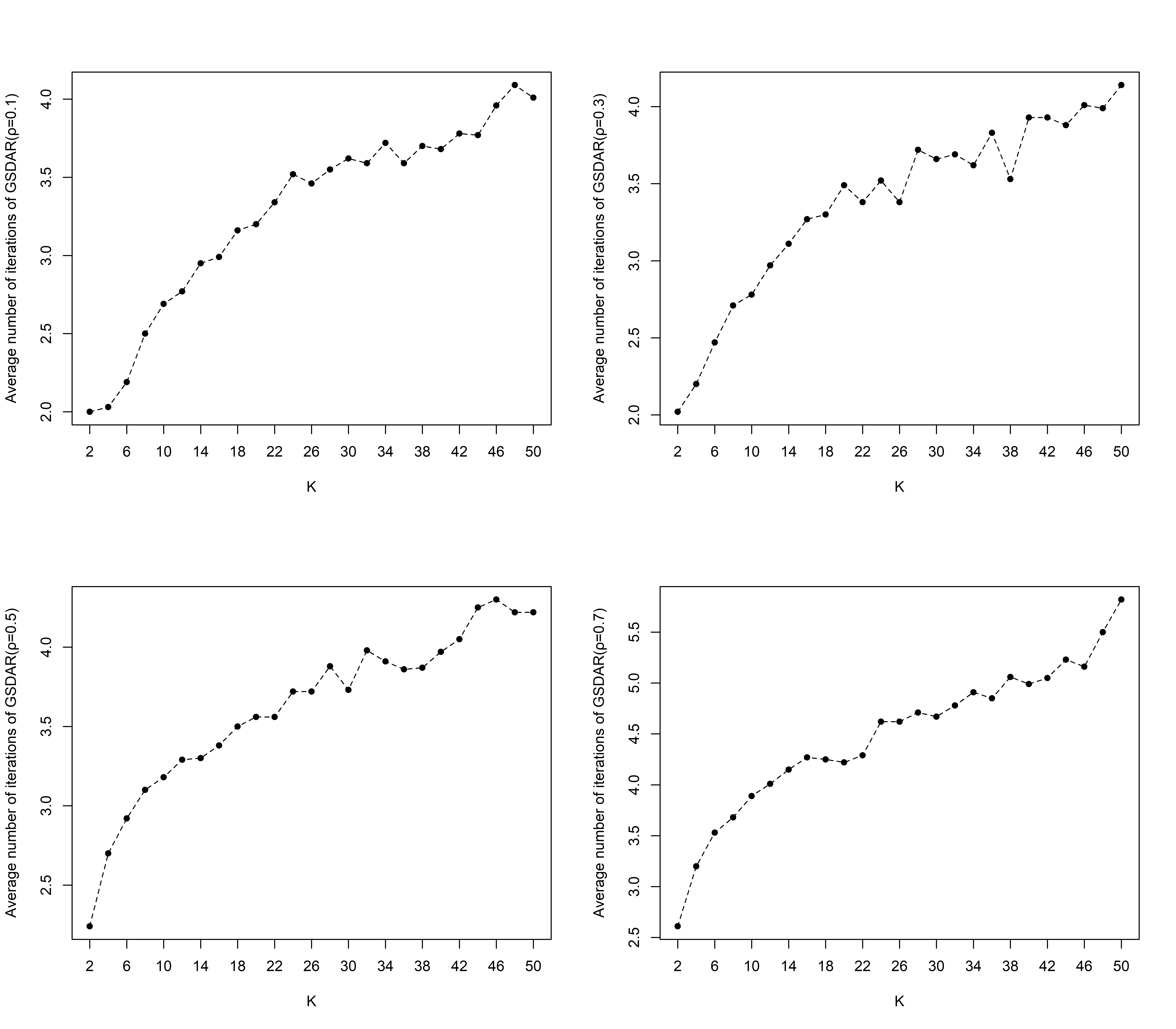}\\
  \caption{The average number of iterations of GSDAR as K increases}\label{F1}
\end{figure}

As shown in  \figref{F1}, the average number of iterations of the GSDAR algorithm increases as the sparsity level increases from 2 to 50 for every $\rho$.
Even the sparsity level $K$ is 50,  the average number of iterations is only 4  when the correlation $\rho$ is 0.1, 0.3, and 0.5, and is nearly 5.5  when the correlation $\rho$ is 0.7. It illustrates  that our approach converges fast.

\begin{table}[!ht]
\centering
\caption{Description of four real data sets}
\label{tabler1}
\scalebox{0.85}{
\begin{tabular}{ccccccccc}
\hline
 Data name  & $n$ samples& $p$ features & training size $n_1$ & testing set $n_2$\\
\hline
colon-cancer&62&2000&62&0\\
duke breast-cancer&42&7129&38&4\\
gisette&7000&5000&6000&1000\\
leukemia&72&7129&38&34\\
\hline
\end{tabular}
}
\end{table}
\begin{table}[!ht]
\centering
\caption{Classification accuracy rate}
\label{tabler2}
\scalebox{1.03}{
\begin{tabular}{ccccccccc}
\hline
 Data name&GSDAR&AGSDAR&Lasso&MCP\\
\hline
colon-cancer&98.39\%&96.77\%&90.32\%&85.48\%\\
duke breast-cancer&1&1&1&25\%\\
gisette&54.10\%&56.30\%&51.30\%&59.90\%\\
leukemia&91.18\%&94.12\%&91.17\%&94.11\%\\
\hline
\end{tabular}
}
\end{table}

\subsection{Real data example}\label{real}
Analysing  biological data using sparse learning methods is a hot topic \cite{gu2017solution,
mahmud2018applications,cai2018modified,li2019generalized,ssn}.
We  demonstrate the performance of the   proposed methods GSDAR and AGSDAR with  four  real data:
colon-cancer, duke breast-cancer, gisette and leukemia, which are exhaustively described in  \tabref{tabler1}
 and can be  downloaded from
\url{https://www.csie.ntu.edu.tw/~cjlin/libsvmtools/datasets/}.
Besides, colon-cancer, duke breast-cancer and leukemia have been normalized such that the mean is $0$ and variance is 1, and the values $-1$s of response variable $y$ are replaced by $0$.
Logistic regression model seeks to classify, then we get the  classification accuracy rate,
 and compare the classification accuracy rate of the proposed methods with Lasso and MCP based on these  real data sets.
Let $T=0.5n/\log(n)$ in GSDAR,
and implement the AGSDAR, Lasso and MCP
by the same way as depicted in \secref{simulation}.
When the data set has no testing data, we get the classification accuracy rate through the training set itself.  The results are showed in \tabref{tabler2},
which indicates that the classification accuracy rates of GSDAR and AGSDAR are comparable to that of Lasso and MCP.
As a result, the prosed methods are effective in
colon-cancer, duke breast-cancer, gisette and leukemia data sets.

\section{Conclusion}\label{conclusion}
We extend the
support detection and root finding (SDAR) algorithm  to estimation in high-dimensional GLMs,
 then we get the GSDAR algorithm.
GSDAR algorithm is also a constructive approach for fitting sparse, high-dimensional GLMs. In theory, we get
$\ell_{\infty}$  optimal error bound for the sequence generated by GSDAR algorithm under some regular conditions.
Further,  we can get the oracle estimator, if the target signal is detectable  with a high probability.
 We propose  the AGSDAR algorithm, one adaptive version of GSDAR, to handle the problem of unknown sparsity level. Numerical results  compared with  Lasso and MCP on simulations and
 real data show  GSDAR algorithm and AGSDAR algorithm are fast and  stable and accurate.

For further research,
we will extend  GSDAR to solve  structured sparsity learning problems \cite{breheny2015group,jiao2017group} with general convex losses
or to problems related to deep neural networks (DNNs)
\cite{scardapane2017group,louizos2018learning,ma2019transformed}.

\section*{Acknowledgements}

The authors are grateful to the anonymous referees, the associate editor and the editor for their helpful comments, which have led to a significant improvement on the quality of the paper.
The work of Jian Huang is supported in part by the NSF grant DMS-1916199.
The work of Y. Jiao was supported in part by
the National Science Foundation of China
under Grant 11871474  and by the research fund
of KLATASDSMOE.
The work of J. Liu is supported by Duke-NUS Graduate Medical School WBS: R913-200-098-263 and MOE2016- T2-2-029 from Ministry of Eduction, Singapore.
The work of Yanyan Liu is   supported in part by
the National Science Foundation of China
under Grant 11971362.
 The work of X. Lu is supported by the National Key Research and Development Program of China (No. 2018YFC1314600), the National Science Foundation of China (No. 91630313 and No. 11871385), and the Natural Science Foundation of Hubei Province (No. 2019CFA007).

\bibliographystyle{abbrv}
\bibliography{STS758pdas}

\newpage

\appendix
\section{Appendix} In the appendix, we will show  the proofs of the theoretical results.
\subsection{Proof of \lemref{L3.1}}
\begin{proof}
Let $L_\lambda(\bbeta)=\mathcal{L}(\bbeta)+\lambda\|\bbeta\|_0$.
 Assume $\bbeta^\diamond$ is a global minimizer of $L_\lambda(\bbeta)$ and $\bd^{\diamond}=-\nabla \mathcal{L}(\bbeta^{\diamond})$. 
Then by  Theorem 10.1  in \cite{rockafellar2009variational}, we have
 \begin{equation}\label{femat}
 \textbf{0} \in \nabla \mathcal{L} (\bbeta^{\diamond}) + \lambda \partial \|\bbeta^{\diamond}\|_{0},
 \end{equation}
 where $\partial \|\bbeta^{\diamond} \|_{0}$ denotes the limiting subdifferential  (see Definition 8.3 in \cite{rockafellar2009variational}) of  $ \|\cdot\|_{0}$ at $\bbeta^{\diamond}$.
 Let $\bd^{\diamond} = -\nabla \mathcal{L}(\bbeta^{\diamond})$ and
 define $G(\bbeta) = \frac{1}{2}\|\bbeta - (\bbeta^{\diamond}+\bd^{\diamond})\|^2 + \lambda \|\bbeta\|_0$.
 Since \eqref{femat} is equivalent to
\begin{equation*}\textbf{0} \in \bbeta^{\diamond} -(\bbeta^{\diamond}+\bd^{\diamond}) + \lambda \partial \|\bbeta^{\diamond}\|_{0},
\end{equation*}
we deduce that  $\bbeta^{\diamond}$ is a KKT point  of $ G(\bbeta)$.  Then
$\bbeta^{\diamond}=H_{\lambda}(\bbeta^{\diamond}+\bd^{\diamond})$ follows from the result that the  KKT poits of $G$ is coincide with its coordinate-wise minimizer \cite{huang2019}.
Conversely,  suppose
$\bbeta^{\diamond}$ and $\bd^{\diamond}$ satisfy \eqref{eq4}, then $\bbeta^{\diamond}$ is a local minimizer of $L_{\lambda}(\bbeta)$. To show $\bbeta^{\diamond}$ is a local minimizer of $L_{\lambda}(\bbeta)$, we can assume $\bh$ is small enough and $\|\bh\|_{\infty}<\sqrt{2\lambda}$. Then we will show $L_{\lambda}(\bbeta^{\diamond}+\bh)\geq L_{\lambda}(\bbeta^{\diamond})$ in two case respectively.\\
\textbf{Case1:} $\bh_{I^{\diamond}}\neq0$.
\begin{equation*}
\|\bbeta^{\diamond}+\bh\|_{0}=\|\bbeta_{A^\diamond}^{\diamond}+\bh_{A^\diamond}\|_{0}+\|\bh_{I^\diamond}\|_{0},
\end{equation*}
\begin{equation*}
\lambda\|\bbeta^{\diamond}+\bh\|_{0}-\lambda\|\bbeta^{\diamond}\|_{0}=\lambda\|\bbeta_{A^\diamond}^{\diamond}+\bh_{A^\diamond}\|_{0}+\lambda\|\bh_{I^\diamond}\|_{0}-\lambda\|\bbeta_{A^\diamond}^{\diamond}\|_{0}.
\end{equation*}
Because $|\beta^{\diamond}_{i}|\geq\sqrt{2\lambda}$  for  $i\in{A^\diamond}$  and  $\|\bh\|_{\infty}<\sqrt{2\lambda}$, we have \begin{equation*}
\lambda\|\bbeta_{A^\diamond}^{\diamond}+\bh_{A^\diamond}\|_{0}-\lambda\|\bbeta_{A^\diamond}^{\diamond}\|_{0}=0,
\end{equation*}
\begin{equation*}
\lambda\|\bbeta^{\diamond}+\bh\|_{0}-\lambda\|\bbeta^{\diamond}\|_{0}=\lambda\|\bh_{I^\diamond}\|_{0}>\lambda.
\end{equation*}
Therefore, we get
\begin{align*}
&L_{\lambda}(\bbeta^{\diamond}+\bh)-L_{\lambda}(\bbeta^{\diamond})\\
&=\sum_{i=1}^{n}[c(\bx_{i}^{T}(\bbeta^{\diamond}+\bh))-c(\bx_{i}^{T}\bbeta^{\diamond})]-\by^T\bX\bh+\lambda\|\bh_{I^\diamond}\|_{0}\\
&>\sum_{i=1}^{n}[c(\bx_{i}^{T}(\bbeta^{\diamond}+\bh))-c(\bx_{i}^{T}\bbeta^{\diamond})]-\by^T\bX\bh+\lambda\\
&>0.
\end{align*}
Let $m(\bh)=\sum_{i=1}^{n}[c(\bx_{i}^{T}(\bbeta^{\diamond}+\bh))-c(\bx_{i}^{T}\bbeta^{\diamond})]-\by^{T}\bX\bh$, so $m(\bh)$ is a continuous function about $\bh$. As $\bh$ is small enough and $\|\bh\|_{\infty}<\sqrt{2\lambda}$, then $m(\bh)+\lambda>0$. Thus the last inequality holds.\\
\textbf{Case2:} $\bh_{I^{\diamond}}=0$. \begin{equation*}
\lambda\|\bbeta^{\diamond}+\bh\|_{0}-\lambda\|\bbeta^{\diamond}\|_{0}=\lambda\|\bbeta_{A^\diamond}^{\diamond}+\bh_{A^\diamond}\|_{0}-\lambda\|\bbeta_{A^\diamond}^{\diamond}\|_{0}.
\end{equation*}
As $|\beta^{\diamond}_{i}|\geq\sqrt{2\lambda}$  for  $i\in{A^\diamond}$ and $\|\bh_{A^\diamond}\|_{\infty}<\sqrt{2\lambda}$, then we have $$\lambda\|\bbeta^{\diamond}+\bh\|_{0}-\lambda\|\bbeta^{\diamond}\|_{0}=\lambda\|\bbeta_{A^\diamond}^{\diamond}+\bh_{A^\diamond}\|_{0}-\lambda\|\bbeta_{A^\diamond}^{\diamond}\|_{0}=0,$$
and
\begin{align*}
&L_{\lambda}(\bbeta^{\diamond}+\bh)-L_{\lambda}(\bbeta^{\diamond})\\
&=\sum_{i=1}^{n}[c(\bx_{i}^{T}(\bbeta^{\diamond}+\bh))-c(\bx_{i}^{T}\bbeta^{\diamond})]-\by^T\bX\bh\\
&=\sum_{i=1}^{n}[c(\bx_{i(A^\diamond)}^{T}(\bbeta_{A^\diamond}^{\diamond}+\bh_{A^\diamond}))-c(\bx_{i(A^\diamond)}^{T}\bbeta_{A^\diamond}^{\diamond})]-\by^T\bX_{A^{\diamond}}\bh_{A^{\diamond}}\\
&=\sum_{i=1}^{n}[c(\bx_{i(A^\diamond)}^{T}(\bbeta_{A^\diamond}^{\diamond}+\bh_{A^\diamond}))]-\by^T\bX_{A^{\diamond}}(\bbeta^{\diamond}_{A^{\diamond}}+ \bh_{A^{\diamond}})\\
&\quad-\sum_{i=1}^{n}[c(\bx_{i(A^\diamond)}^{T}\bbeta_{A^\diamond}^{\diamond})]+\by^T\bX_{A^{\diamond}}\bbeta^{\diamond}_{A^{\diamond}}\\
&\geq0.
\end{align*}
As known that $\bbeta_{A^\diamond}^{\diamond}\in\underset{\bbeta_{A^\diamond}}{\mbox{argmin}}~\widetilde{\mathcal{L}}(\bbeta_{A^\diamond})$, so the last inequality holds.
In summary, $\bbeta^{\diamond}$ is a local minimizer of $L_{\lambda}(\bbeta^{\diamond})$.
\end{proof}
\begin{lemma}\label{L9.1}
Assume  (C1) holds and $\|\bbeta^{*}\|_{0}=K\leq T$.  Denote $B^k = A^{k}\backslash A^{k-1}$. Then,
\begin{equation*}
\|\nabla_{B^k}\mathcal{L}(\bbeta^{k})\|_1\|\nabla_{B^k}\mathcal{L}(\bbeta^{k})\|_{\infty}\geq 2L\zeta[\mathcal{L}(\bbeta^k)-\mathcal{L}(\bbeta^*)],
\end{equation*}
where $\zeta=\frac{|B^k|}{|B^k|+|A^*\backslash A^{k-1}|}$.
\end{lemma}
\begin{proof}[Proof]
Obviously, this lemma holds if $A^{k}=A^{k-1}$ or $\mathcal{L}(\bbeta^k)\leq \mathcal{L}(\bbeta^*)$. So we only prove the lemma by assuming $A^{k}\neq A^{k-1}$ and $\mathcal{L}(\bbeta^k)>\mathcal{L}(\bbeta^*)$. The condition (C1) indicates
\begin{align*}
&\mathcal{L}(\bbeta^{*})-\mathcal{L}(\bbeta^k)-\langle\nabla \mathcal{L}(\bbeta^k),{\bbeta^*-\bbeta^k}\rangle\\
&\quad\geq\frac{L}{2}\big{\|}\bbeta^*-\bbeta^k\big{\|}_1\big{\|}\bbeta^*-\bbeta^k\big{\|}_{\infty}.
\end{align*}
Hence,
\begin{align*}
&\langle-\nabla\mathcal{L}(\bbeta^k),{\bbeta^*-\bbeta^k}\rangle\\
&=\langle\nabla \mathcal{L}(\bbeta^k),-\bbeta^*\rangle\\
&\geq\frac{L}{2}\big{\|}\bbeta^*-\bbeta^k\big{\|}_1\big{\|}\bbeta^*-\bbeta^k\big{\|}_{\infty}+\mathcal{L}(\bbeta^k)-\mathcal{L}(\bbeta^{*})\\
&\geq \sqrt{2L}\sqrt{\big{\|}\bbeta^*-\bbeta^k\big{\|}_1\big{\|}\bbeta^*-\bbeta^k\big{\|}_{\infty}}\sqrt{\mathcal{L}(\bbeta^k)-\mathcal{L}(\bbeta^{*})}.\\
\end{align*}
From the definition of $A^{k}$ and $A^*$, it is known that 
$B^k$ contains the first
$|B^k|$-largest elements (in absolute value) of $\nabla \mathcal{L}(\bbeta^k)$, and
$\mbox{supp}(\nabla \mathcal{L}(\bbeta^k))\bigcap \mbox{supp}(\bbeta^*)=A^{*}\backslash A^{k-1}$.
Thus,
we have
\begin{align*}
\langle\nabla \mathcal{L}(\bbeta^k),-\bbeta^*\rangle
&\leq\frac{1}{\sqrt{\zeta}}\|\nabla_{B^k}\mathcal{L}(\bbeta^k)\|_2\|\bbeta_{A^*\backslash A^{k-1}}^{*}\|_2\\
&=\frac{1}{\sqrt{\zeta}}\|\nabla_{B^k}\mathcal{L}(\bbeta^k)\|_2\|(\bbeta^{*}-\bbeta^k)_{A^*\backslash A^{k-1}}\|_2\\
&\leq\frac{1}{\sqrt{\zeta}}\|\nabla_{B^k}\mathcal{L}(\bbeta^k)\|_2\|\bbeta^{*}-\bbeta^k\|_2\\
&\leq\frac{1}{\sqrt{\zeta}}\sqrt{\|\nabla_{B^k}\mathcal{L}(\bbeta^k)\|_1\|\nabla_{B^k}\mathcal{L}(\bbeta^k)\|_{\infty}}\\
&\quad\quad\times\sqrt{\|\bbeta^{*}-\bbeta^k\|_1\|\bbeta^{*}-\bbeta^k\|_{\infty}}.
\end{align*}
Therefore,
 {\footnotesize
\begin{equation*}
\sqrt{2L}\sqrt{\mathcal{L}(\bbeta^k)-\mathcal{L}(\bbeta^{*})}\leq\frac{1}{\sqrt{\zeta}}\sqrt{\|\nabla_{B^k}\mathcal{L}(\bbeta^k)\|_1\|\nabla_{B^{k}}\mathcal{L}(\bbeta^k)\|_{\infty}}.
\end{equation*}
}
In summary,
\begin{equation*}
\|\nabla_{B^k}\mathcal{L}(\bbeta^{k})\|_1\|\nabla_{B^k}\mathcal{L}(\bbeta^{k})\|_{\infty}\geq 2L\zeta[\mathcal{L}(\bbeta^k)-\mathcal{L}(\bbeta^*)].
\end{equation*}
\end{proof}
\begin{lemma}\label{L9.2}
Assume (C1) holds with $0<U<\frac{1}{T}$,  and
$K\leq T$ in  \algref{algorithm1}. Then before  \algref{algorithm1} terminates,
\begin{equation*}
\mathcal{L}(\bbeta^{k+1})-\mathcal{L}(\bbeta^*)\leq\xi[\mathcal{L}(\bbeta^k)-\mathcal{L}(\bbeta^*)],
\end{equation*}
where $\xi=1-\frac{2L(1-TU)}{T(1+K)}\in(0,1)$.
\end{lemma}
\begin{proof}[Proof]
Let $\bb^{k}=\bbeta^k-\nabla \mathcal{L}(\bbeta^{k})$. The condition of (C1) indicates
{\footnotesize
\begin{align*}
&\mathcal{L}(\bb^{k+1}|_{A^{k+1}})-\mathcal{L}(\bbeta^{k+1})\leq\langle\nabla \mathcal{L}(\bbeta^{k+1}),\bb^{k+1}|_{A^{k+1}}-\bbeta^{k+1}\rangle\\
&\quad\quad\quad\quad+\frac{U}{2}\big{\|}\bb^{k+1}|_{A^{k+1}}-\bbeta^{k+1}\big{\|}_1\big{\|}\bb^{k+1}|_{A^{k+1}}-\bbeta^{k+1}\big{\|}_{\infty}.
\end{align*}
}
On the one hand, by the definition of $\bbeta^{k+1}$ and $\nabla \mathcal{L}(\bbeta^{k+1})$, we have
\begin{align*}
&\langle\nabla \mathcal{L}(\bbeta^{k+1}), \bb^{k+1}|_{A^{k+1}}-\bbeta^{k+1}\rangle\\
&=\langle\nabla \mathcal{L}(\bbeta^{k+1}), \bb^{k+1}|_{A^{k+1}}\rangle\\
&=\langle\nabla_{A^{k+1}}\mathcal{L}(\bbeta^{k+1}), \bb_{A^{k+1}}^{k+1}\rangle\\
&=\langle\nabla_{A^{k+1}\backslash A^{k}}\mathcal{L}(\bbeta^{k+1}), \bb_{A^{k+1}\backslash A^{k}}^{k+1}\rangle.
\end{align*}
Further, we also have
\begin{align*}
&\big{\|}\bb^{k+1}|_{A^{k+1}}-\bbeta^{k+1}\big{\|}_1\\
&=\big{\|}\bb^{k+1}|_{A^{k+1}\backslash A^{k}}+\bb^{k+1}|_{A^{k+1}\bigcap A^{k}}\\
&\quad\quad-\bbeta^{k+1}|_{A^{k+1}\bigcap A^{k}}-\bbeta^{k+1}|_{A^{k}\backslash A^{k+1}}\big{\|}_1\\
&=\big{\|}\bb_{A^{k+1}\backslash A^{k}}^{k+1}\big{\|}_1+\big{\|}\bb_{A^{k+1}\bigcap A^{k}}^{k+1}-\bbeta_{A^{k+1}\bigcap A^{k}}^{k+1}\big{\|}_1\\
&\quad\quad+\big{\|}\bbeta_{A^{k}\backslash A^{k+1}}^{k+1}\big{\|}_1\\
&=\big{\|}\bb_{A^{k+1}\backslash A^{k}}^{k+1}\big{\|}_1+\big{\|}\bbeta_{A^{k}\backslash A^{k+1}}^{k+1}\big{\|}_1,
\end{align*}
and
\begin{align*}
&\big{\|}\bb^{k+1}|_{A^{k+1}}-\bbeta^{k+1}\big{\|}_{\infty}\\
&=\big{\|}\bb^{k+1}|_{A^{k+1}\backslash A^{k}}+\bb^{k+1}|_{A^{k+1}\bigcap A^{k}}\\
&\quad\quad-\bbeta^{k+1}|_{A^{k+1}\bigcap A^{k}}-\bbeta^{k+1}|_{A^{k}\backslash A^{k+1}}\big{\|}_{\infty}\\
&=\big{\|}\bb_{A^{k+1}\backslash A^{k}}^{k+1}\big{\|}_{\infty}\bigvee\big{\|}\bbeta_{A^{k}\backslash A^{k+1}}^{k+1}\big{\|}_{\infty},
\end{align*}
where
$a \bigvee b=\max\{a,b\}$.
On the other hand, by the definition of $A^k$, $A^{k+1}$ and $\bbeta^{k+1}$, we know that
\begin{equation*}
|A^{k}\backslash A^{k+1}|=|A^{k+1}\backslash A^{k}|,\quad \bb_{A^{k}\backslash A^{k+1}}^{k+1}=\bbeta_{A^{k}\backslash A^{k+1}}^{k+1}.
\end{equation*}
By the definition of $A^{k+1}$, we can conclude that
\begin{equation*}
\|\bb_{A^{k}\backslash A^{k+1}}^{k+1}\|_1=\|\bbeta_{A^{k}\backslash A^{k+1}}^{k+1}\|_1\leq \|\bb_{A^{k+1}\backslash A^{k}}^{k+1}\|_1,
\end{equation*}
\begin{equation*}
\|\bb_{A^{k+1}\backslash A^{k}}^{k+1}\|_{\infty}\bigvee\|\bbeta_{A^{k}\backslash A^{k+1}}^{k+1}\|_{\infty}=\|\bb_{A^{k+1}\backslash A^{k}}^{k+1}\|_{\infty}.
\end{equation*}
Due to $-\nabla_{A^{k+1}\backslash A^{k}}\mathcal{L}(\bbeta^{k+1})=\bb_{A^{k+1}\backslash A^{k}}^{k+1}$ and $U<\frac{1}{T}$, hence we can deduce that
\begin{align*}
&\mathcal{L}(\bb^{k+1}|_{A^{k+1}})-\mathcal{L}(\bbeta^{k+1})\\
&\leq\langle\nabla_{A^{k+1}\backslash A^{k}}\mathcal{L}(\bbeta^{k+1}),\bb_{A^{k+1}\backslash A^{k}}^{k+1}\rangle\\
&\quad+U\big{\|}\bb_{A^{k+1}\backslash A^{k}}^{k+1}\big{\|}_1\big{\|}\bb_{A^{k+1}\backslash A^{k}}^{k+1}\big{\|}_{\infty}\\
&\leq-(1/T-U)\big{\|}\nabla_{A^{k+1}\backslash A^{k}}\mathcal{L}(\bbeta^{k+1})\big{\|}_1\\
&\quad\times\big{\|}\nabla_{A^{k+1}\backslash A^{k}}\mathcal{L}(\bbeta^{k+1})\big{\|}_{\infty}.
\end{align*}
By the definition of $\bbeta^{k+1}$, we can get
\begin{align*}
&\mathcal{L}(\bbeta^{k+1})-\mathcal{L}(\bbeta^{k})\\
&\leq \mathcal{L}(\bb^{k}|_{A^{k}})-\mathcal{L}(\bbeta^{k})\\
&\leq-(1/T-U)\big{\|}\nabla_{A^{k+1}\backslash A^{k}}\mathcal{L}(\bbeta^{k+1})\big{\|}_1\\
&\quad\times\big{\|}\nabla_{A^{k+1}\backslash A^{k}}\mathcal{L}(\bbeta^{k+1})\big{\|}_{\infty}.
\end{align*}
Moreover, $\frac{|A^*\backslash A^{k-1}|}{|B^k|}\leq K$.
By  \lemref{L9.1}, we have
\begin{equation*}
\mathcal{L}(\bbeta^{k+1})-\mathcal{L}(\bbeta^{k})\leq-\frac{2L(1-TU)}{T(1+K)}[\mathcal{L}(\bbeta^{k})-\mathcal{L}(\bbeta^{*})].
\end{equation*}

Therefore, we have $$\mathcal{L}(\bbeta^{k+1})-\mathcal{L}(\bbeta^*)\leq\xi[\mathcal{L}(\bbeta^k)-\mathcal{L}(\bbeta^*)],$$
where $\xi=1-\frac{2L(1-TU)}{T(1+K)}\in(0,1)$.
\end{proof}
\begin{lemma}\label{L9.3}
Assume $\mathcal{L}$ satisfies (C1)
and \begin{equation*}
\mathcal{L}(\bbeta^{k+1})-\mathcal{L}(\bbeta^{*})\leq \xi[\mathcal{L}(\bbeta^{k})-\mathcal{L}(\bbeta^{*})]
\end{equation*}
for all $k\geq0$.
Then,
\begin{equation}\label{eq11}
\begin{split}
\|\bbeta^{k}-\bbeta^{*}\|_{\infty}&\leq
\sqrt{(K+T)(1+\frac{U}{L})}(\sqrt{\xi})^k\|\bbeta^0-\bbeta^*\|_{\infty}\\
&\quad+\frac{2}{L}\|\nabla \mathcal{L}(\bbeta^*)\|_{\infty}.
\end{split}
\end{equation}
\end{lemma}
\begin{proof}[Proof]
If $\|\bbeta^k-\bbeta^*\|_{\infty}< \frac{2\|\nabla \mathcal{L}(\bbeta^*)\|_{\infty}}{L}$, then \eqref{eq11} holds, so we only consider  the case that $\|\bbeta^k-\bbeta^*\|_{\infty}\geq \frac{2\|\nabla
\mathcal{L}(\bbeta^*)\|_{\infty}}{L}$.
On the one hand,
$\mathcal{L}$ satisfies (C1), then
\begin{align*}
&\mathcal{L}(\bbeta^k)-\mathcal{L}(\bbeta^*)\\
&\geq\langle\nabla \mathcal{L}(\bbeta^*),\bbeta^k-\bbeta^*\rangle+\frac{L}{2}\big{\|}\bbeta^k-\bbeta^*\big{\|}_{1}\big{\|}\bbeta^k-\bbeta^*\big{\|}_{\infty}\\
&\geq-\|\nabla \mathcal{L}(\bbeta^*)\|_{\infty}\|\bbeta^k-\bbeta^*\|_1+\frac{L}{2}\big{\|}\bbeta^k-\bbeta^*\big{\|}_1\big{\|}\bbeta^k-\bbeta^*\big{\|}_{\infty}.
\end{align*}
Due to $\|\bbeta^k-\bbeta^*\|_{\infty}\geq \frac{2\|\nabla
\mathcal{L}(\bbeta^*)\|_{\infty}}{L}$,
then
 {\footnotesize
\begin{equation*}
(\|\bbeta^k-\bbeta^*\|_1-\|\bbeta^k-\bbeta^*\|_{\infty})(\frac{L}{2}\|\bbeta^k-\bbeta^*\|_{\infty}-\|\nabla \mathcal{L}(\bbeta^*)\|_{\infty})\geq 0.
\end{equation*}
}
Further, we can get
 {\footnotesize
\begin{equation*}
\frac{L}{2}\|\bbeta^k-\bbeta^*\|_{\infty}^2-\|\nabla \mathcal{L}(\bbeta^*)\|_{\infty}\|\bbeta^k-\bbeta^*\|_{\infty}-[\mathcal{L}(\bbeta^k)-\mathcal{L}(\bbeta^*)]\leq0,
\end{equation*}
}
which is univariate quadratic inequality about $\|\bbeta^k-\bbeta^*\|_{\infty}$.
Thus, by simple computation, we can get
 {\footnotesize
\begin{equation}\label{eq12}
\|\bbeta^k-\bbeta^*\|_{\infty}\leq \sqrt{\frac{2\max\{\mathcal{L}(\bbeta^k)-\mathcal{L}(\bbeta^*),0\}}{L}}+\frac{2\|\nabla \mathcal{L}(\bbeta^*)\|_{\infty}}{L}.
\end{equation}
}
On the other hand,
because $\mathcal{L}$ satisfies (C1), then
\begin{align*}
&\mathcal{L}(\bbeta^0)-\mathcal{L}(\bbeta^*)\\
&\leq\langle\nabla \mathcal{L}(\bbeta^*),\bbeta^0-\bbeta^*\rangle+\frac{U}{2}\big{\|}\bbeta^0-\bbeta^*\big{\|}_1\big{\|}\bbeta^0-\bbeta^*\big{\|}_{\infty}\\
&\leq\big{\|}\nabla \mathcal{L}(\bbeta^*)\big{\|}_{\infty}\big{\|}\bbeta^0-\bbeta^*\big{\|}_1+\frac{U}{2}\big{\|}\bbeta^0-\bbeta^*\big{\|}_1\big{\|}\bbeta^0-\bbeta^*\big{\|}_{\infty}\\
&\leq (K+T)\big{\|}\bbeta^0-\bbeta^{*}\big{\|}_{\infty}(\|\nabla \mathcal{L}(\bbeta^*)\big{\|}_{\infty}+\frac{U}{2}\big{\|}\bbeta^0-\bbeta^*\big{\|}_{\infty}).
\end{align*}
Then, we can get
\begin{align*}
\mathcal{L}(\bbeta^{k})-\mathcal{L}(\bbeta^{*})
&\leq \xi[\mathcal{L}(\bbeta^{k-1})-\mathcal{L}(\bbeta^{*})]\\
&\leq \xi^{k}[\mathcal{L}(\bbeta^{0})-\mathcal{L}(\bbeta^{*})]\\
&\leq\xi^k(K+T)\big{\|}\bbeta^0-\bbeta^{*}\big{\|}_{\infty}\\
&\quad\times(\|\nabla \mathcal{L}(\bbeta^*)\big{\|}_{\infty}+\frac{U}{2}\big{\|}\bbeta^0-\bbeta^*\big{\|}_{\infty})\\
&\leq
\frac{\xi^k (L+U)(K+T)}{2}\big{\|}\bbeta^0-\bbeta^{*}\big{\|}_{\infty}^2.
\end{align*}
Hence, by \eqref{eq12}, we have
\begin{align*}
\|\bbeta^{k}-\bbeta^{*}\|_{\infty}&\leq
\sqrt{(K+T)(1+\frac{U}{L})}(\sqrt{\xi})^k\|\bbeta^0-\bbeta^*\|_{\infty}\\
&\quad+\frac{2}{L}\|\nabla \mathcal{L}(\bbeta^*)\|_{\infty}.
\end{align*}
\end{proof}
\begin{lemma}\label{L1} (Proof of Corollary 2 in \cite{loh2015regularized}).
Assume $x_{ij}^{,}s$ are sub-Gaussian and $n \gtrsim \log(p)$, then
there exists universal constants $(c_1,c_2,c_3)$ with $0<c_i<\infty$, $i=1,2,3$ such that
\begin{equation*}
P\left(\|\nabla \mathcal{L}(\bbeta^*)\|_{\infty}\geq c_1\sqrt{\frac{\log(p)}{n}}\right)\leq c_2\exp(-c_3\log(p)).
\end{equation*}
\end{lemma}
\subsection{Proof of \thmref{th1}}
\begin{proof}
By  \lemref{L9.2}, we have \begin{equation*}
\mathcal{L}(\bbeta^{k+1})-\mathcal{L}(\bbeta^*)\leq\xi[\mathcal{L}(\bbeta^k)-\mathcal{L}(\bbeta^*)],
\end{equation*}
where
\begin{equation*}
\xi=1-\frac{2L(1-TU)}{T(1+K)}\in(0,1).
\end{equation*}
So the conditions of  \lemref{L9.3} are satisfied. Taking $\bbeta^0 = \textbf{0}$, we can get
\begin{align*}
&\|\bbeta^{k}-\bbeta^{*}\|_{\infty}\\
&\leq
\sqrt{(K+T)(1+\frac{U}{L})}(\sqrt{\xi})^k\|\bbeta^*\|_{\infty}
+\frac{2}{L}\|\nabla \mathcal{L}(\bbeta^*)\|_{\infty}.
\end{align*}
By  \lemref{L1}, then there exists universal constants $(c_1,c_2,c_3)$ defined in \lemref{L1}, with at least probability $1-c_2\exp(-c_3\log(p))$, we have
\begin{align}\label{errorlinf}
&\|\bbeta^{k}-\bbeta^{*}\|_{\infty}\nonumber\\
&\leq
\sqrt{(K+T)(1+\frac{U}{L})}(\sqrt{\xi})^k\|\bbeta^*\|_{\infty}
+\frac{2c_1}{L} \sqrt{\frac{\log(p)}{n}}.
\end{align}
Some algebra shows that
\begin{equation*}
\|\bbeta^{k}-\bbeta^{*}\|_{\infty}\leq \mathcal{O}(\sqrt{\frac{\log(p)}{n}})
\end{equation*}
 by taking  $k \geq \mathcal{O}(\log_{\frac{1}{\xi}} \frac{n}{\log(p)} )$  in \eqref{errorlinf}.
Then, the proof is complete.
\end{proof}

\subsection{Proof of \thmref{th3}}
\begin{proof}
\eqref{errorlinf} and assumption (C2) and some algebra shows that  that
\begin{align*}
&\|\bbeta^{k}-\bbeta^{*}\|_{\infty}\\
&\leq
\sqrt{(K+T)(1+\frac{U}{L})}(\sqrt{\xi})^k\|\bbeta^*\|_{\infty}
+\frac{2}{3}\|\bbeta^*_{A^*}\|_{\min}\\
&< \|\bbeta^*_{A^*}\|_{\min},
\end{align*}
if $k>\log_{\frac{1}{\xi}} 9 (T+K)(1+\frac{U}{L})r^2.$
This implies  that $A^* \subseteq A^k$.
\end{proof}

\end{document}